\documentclass[11pt, oneside]{article}   	
\usepackage{geometry}                		
\usepackage{graphicx}				
\usepackage{amssymb}

\usepackage{amssymb}
\usepackage{float}
\usepackage{amssymb}
\usepackage{amsmath}
\usepackage{amsthm}
\usepackage{color}
\usepackage{graphicx}
\usepackage{hyperref}
\usepackage{fullpage}
\usepackage{graphicx}
\usepackage{mathtools}
\usepackage{microtype}
\usepackage{bbm}
\usepackage{booktabs}

\newtheorem{theorem}{Theorem}

\title{A Note on Optimizing the Ratio of Monotone Supermodular Functions}
\author{Wenxin Li\\
The Ohio State University\\
{\tt wenxinliwx.1@gmail.com}\\
{li.7328@osu.edu}\\	
}
\begin{document}
\maketitle
\begin{abstract}
We show that for the problem of minimizing (or maximizing) the ratio of two supermodular functions, no bounded approximation ratio can be achieved via polynomial number of queries, if the two supermodular functions are both monotone non-decreasing or non-increasing.
\end{abstract}
A set function $f:2^{E}\rightarrow \mathbb{R}^{+}$ defined on ground $E$ of size $n$ is \emph{supermodular}, if inequality $f(S)+f(T)\leq f(S\cup T)+f(S\cap T)$ holds for any two subsets $S,T\subseteq E$. In this paper we consider the problem of optimizing the ratio of two monotone supermodular functions, which is proposed in~\cite{bai2016algorithms}. Formally we are given two monotone supermodular functions $f: 2^{E}\rightarrow \mathbb{R}^{+}$ and $g: 2^{E}\rightarrow \mathbb{R}^{+}$, the objective is to optimize set function $h: 2^{E}\rightarrow \mathbb{R}^{+}$, where $h(S)=f(S)/g(S)$ for $\forall S\subseteq E$. 

Our main result is stated in the following theorem. It is worth noting that the problem is trivial if one of the two supermodular functions is monotone non-decreasing while the other one is monotone non-increasing, since the objective $h$ is monotone in this case. 

\begin{theorem}
No polynomial time algorithm can achieve bounded approximation ratio for optimizing the ratio of two supermodular functions, if the two supermodular functions are both monotone non-decreasing or non-increasing. \end{theorem}
\begin{proof}
We first note that it suffices to prove the result for minimization problem, since the algorithm for maximizing $g/h$ can be used to obtain a solution for minimizing $h=f/g$. Hence in the following we focus on the problem of minimizing function $h$.

When both $f$ and $g$ are monotone non-increasing, we consider the following instance, which is a simple modification of that in~\cite{svitkina2008submodular,goemans2009approximating,bai2016algorithms},
\begin{align*}
f(S)&=\alpha+\varepsilon-\min\{\alpha, |S|\}\\
g_{R}(S)&=\alpha+\varepsilon-\min\{\beta+|S\cap \bar{R}|,\alpha,|S|\},
\end{align*}
where $\varepsilon>0$, $\alpha=x\sqrt{n}/5$, $\beta=x^{2}/5$ for $x=\omega(\log n)$, $R$ is a random set with cardinality $\alpha$.  Note that $f(S)\neq g_{R}(S)$ if and only if $\beta+|S\cap \bar{R}|<\min\{\alpha, |S|\}$, via similar arguments to~\cite{svitkina2008submodular,goemans2009approximating,bai2016algorithms}, we can show that one cannot distinguish $f$ and $g_{R}$ in polynomial number of queries (together with the returned solution set). On the other hand, note that $f(R)/g_{R}(R)=\frac{\varepsilon}{\alpha+\varepsilon-\beta}\leq \varepsilon$, while $h$ achieves an objective value of $1$ at the returned set, this implies that the approximation ratio of a polynomial time algorithm could be arbitrary bad when $\varepsilon \rightarrow 0$.

When both $f$ and $g$ are monotone non-decreasing, consider functions
\begin{align*}
f(S)=
\begin{cases}
|S|& \text{$|S|\leq \lfloor \frac{n}{2}\rfloor$}\\
m\cdot 2^{|S|+1}+|S| & \text{$|S|>\lfloor \frac{n}{2}\rfloor$}
\end{cases}	
\end{align*}
and 
\begin{align*}
g(S)=
\begin{cases}
\frac{2|S|}{n}\cdot \varepsilon& \text{$|S|\leq \lfloor \frac{n}{2}\rfloor$}\\
2(|S|- \lfloor \frac{n}{2}\rfloor) & \text{$|S|>\lfloor \frac{n}{2}\rfloor$},
\end{cases}	
\end{align*}
then the ratio of $f$ and $g$ is given as 
\begin{align*}
h(S)=\frac{f(S)}{g(S)}=\begin{cases}
\frac{n}{2\varepsilon}& \text{$|S|\leq \lfloor \frac{n}{2}\rfloor$}\\
\frac{m\cdot 2^{|S|+1}+|S|}{2(|S|- \lfloor \frac{n}{2}\rfloor)} & \text{$|S|>\lfloor \frac{n}{2}\rfloor$}.
\end{cases}	
\end{align*}
It can be verified that both $f$ and $g$ are monotone non-decreasing supermodular functions when $m>0$ and $\varepsilon\in (0,\frac{n}{n+2}]$. Furthermore, for any set $S$ we have
\begin{align*}
h(S)\geq \min\Big\{\frac{n}{2\varepsilon},m\Big\}.
\end{align*}
Let $R$ be a random set with size $\lfloor \frac{n}{2}\rfloor$. We further consider the following supermodular functions 
\begin{align*}
f_{R}(S)=f(S)
\end{align*}
and
\begin{align*}
g_{R}(S)=
\begin{cases}
\frac{2|S|}{n}\cdot \varepsilon& \text{$|S|\leq \lfloor \frac{n}{2}\rfloor$}\\
1& \text{$S=R$}\\
2(|S|- \lfloor \frac{n}{2}\rfloor) & \text{$|S|>\lfloor \frac{n}{2}\rfloor$},
\end{cases}	
\end{align*}
then $h_{R}(R)=|R|=\lfloor \frac{n}{2}\rfloor$. Let $\mathcal{Q}=\{S_{1}, S_{2},\ldots, S_{|\mathcal{Q}|}\}$ be the query sequence issued by a polynomial time algorithm. As there are exponential number of subsets with size $\lfloor \frac{n}{2}\rfloor$, there must exist $R^{*}$ such that $f(S_{i})=f_{R^{*}}(S_{i})$ and $g(S_{i})=g_{R^{*}}(S_{i})$ for $\forall i\in[|\mathcal{Q}|]$ since $|\mathcal{Q}|$ is polynomial in the size of ground set. As a consequence, functions $h(\cdot)$ and $h_{R^{*}}(\cdot)$ are indistinguishable under query set $\mathcal{Q}$, the corresponding gap in the minimum objective value provides a lower bound on the approximation ratio:
\begin{align*}
\frac{\min h(\cdot)}{\min h_{R^{*}}(\cdot)}\geq \frac{\min\{\frac{n}{2\varepsilon},m\}}{h_{R^{*}}(R^{*})}\geq \min\Big\{\frac{1}{\varepsilon},\frac{2m}{n}\Big\},
\end{align*}
which approaches infinity when $\varepsilon \rightarrow 0$ and $m \rightarrow \infty$. The proof is complete.
\end{proof}
\bibliographystyle{plain}	
\bibliography{paper}
\end{document}